\documentclass[11pt]{article}
\usepackage[utf8]{inputenc}
\usepackage{amsmath, amssymb, amsthm}
\newtheorem{proposition}{Proposition}
\usepackage{mathtools}
\usepackage{bm}
\usepackage{enumitem}
\usepackage{geometry}
\usepackage{hyperref}
\usepackage{algorithm}
\usepackage{algpseudocode}
\usepackage{amsthm} 
\theoremstyle{definition}
\newtheorem{definition}{Definition}[section]  
\newtheorem{theorem}{Theorem}

\hypersetup{colorlinks=true, linkcolor=blue, citecolor=blue, urlcolor=blue}

\title{\textbf{EDGE: A Theoretical Framework for Misconception-Aware Adaptive Learning}\\
\large Evaluate $\rightarrow$ Diagnose $\rightarrow$ Generate $\rightarrow$ Exercise}
\author{Ananda Prakash Verma\\ \texttt{anandaprakashverma@gmail.com}\\}
\date{\today}
\begin{document}
\maketitle

\begin{abstract}
We present \textsc{EDGE}, a general-purpose, misconception-aware adaptive learning framework composed of four stages: \emph{Evaluate} (ability and state estimation), \emph{Diagnose} (posterior inference of misconceptions), \emph{Generate} (counterfactual item synthesis), and \emph{Exercise} (index-based retrieval scheduling). \textsc{EDGE} unifies psychometrics (IRT/Bayesian state-space models), cognitive diagnostics (misconception discovery from distractor patterns and response latencies), contrastive item generation (minimal perturbations that invalidate learner shortcuts while preserving psychometric validity), and principled scheduling (a restless bandit approximation to spaced retrieval). We formalize a composite readiness metric, \emph{EdgeScore}, prove its monotonicity and Lipschitz continuity, and derive an index policy that is near-optimal under mild assumptions on forgetting and learning gains. We further establish conditions under which counterfactual items provably reduce the posterior probability of a targeted misconception faster than standard practice. The paper focuses on theory and implementable pseudocode; empirical study is left to future work.
\end{abstract}

\section{Introduction}
Large-scale preparation for high-stakes exams requires efficient remediation of \emph{misconceptions}---stable, systematic errors not addressed by generic practice. Classical spaced repetition and IRT-based selection optimize recall and information gain, yet they typically remain \emph{misconception-agnostic}. \textsc{EDGE} operationalizes a pipeline that (i) estimates learner state continuously, (ii) infers latent misconception structure from the joint distribution of correctness, chosen distractors, time, and confidence, (iii) synthesizes contrastive counterfactual items that fail the learner's current shortcut, and (iv) schedules retrieval to jointly maximize mastery and retention subject to time budgets. We provide a complete mathematical treatment and design-ready algorithms suitable for deployment at scale.

\paragraph{Contributions.} (1) A probabilistic state model that augments 2PL-IRT with time- and confidence-aware likelihoods. (2) A Bayesian diagnostic for misconception posteriors using clusterable distractor embeddings and timing signatures. (3) A constrained optimization view of counterfactual item generation that preserves construct validity. (4) A restless-bandit scheduler with a closed-form priority index and theoretical properties. (5) A composite readiness score, \emph{EdgeScore}, with calibration and guarantees.

\section{Preliminaries and Notation}
Let $\mathcal{U}$ be learners, $\mathcal{T}$ topics, and $\mathcal{Q}$ items. Each item $q$ is associated with a difficulty $b_q\in\mathbb{R}$, discrimination $a_q>0$, and a concept load vector $\mathbf{c}_q\in\{0,1\}^{|\mathcal{T}|}$ (\emph{Q-matrix}). For learner $u$, denote topic ability vector $\bm{\theta}_u\in\mathbb{R}^{|\mathcal{T}|}$ and per-topic retention variables $\bm{\rho}_u\in[0,1]^{|\mathcal{T}|}$. Let $y_{u,q}\in\{0,1\}$ be correctness, $d_{u,q}\in\{1,\dots,D_q\}$ the distractor index if $y=0$, $\tau_{u,q}\ge0$ response time, and $s_{u,q}\in[0,1]$ self-reported confidence.

\subsection{Ability and Response Model}
We adopt a 2PL likelihood with time and confidence effects. For item $q$ covering topic set $S(q)=\{t: c_{q,t}=1\}$, define the effective ability
\begin{equation}
\theta^{\text{eff}}_{u,q} \triangleq \frac{\sum_{t\in S(q)} w^{(M)}_t \theta_{u,t}}{\sum_{t\in S(q)} w^{(M)}_t},
\end{equation}
with nonnegative weights $w^{(M)}_t$. The probability of correctness is
\begin{equation}
P(y_{u,q}=1\mid \theta^{\text{eff}}_{u,q}, a_q, b_q, \tau_{u,q}, s_{u,q})=\sigma\!\left(a_q(\theta^{\text{eff}}_{u,q}-b_q) + \beta_\tau g(\tau_{u,q}) + \beta_s h(s_{u,q})\right),
\label{eq:irt}
\end{equation}
where $\sigma(x)=1/(1+e^{-x})$, $g$ is a bounded, decreasing function (faster correct $\Rightarrow$ larger logit), and $h$ is increasing (confidence interacts with correctness via the likelihood below).

\subsection{State-Space Update}
Let the prior over $\bm{\theta}_u$ be Gaussian and independent across topics: $\theta_{u,t}\sim \mathcal{N}(\mu_{u,t},\sigma^2_{u,t})$. After observing $(y,\tau,s)$ for item $q$, a Laplace update yields
\begin{align}
\mu'_{u,t}&=\mu_{u,t} + \kappa_{u,q,t}\cdot \left(y - P(y=1\mid \cdot)\right),\\
\sigma'^{2}_{u,t}&=\left(\sigma^{-2}_{u,t}+ a_q^2 \cdot \eta_{u,q,t}\right)^{-1},\qquad t\in S(q),
\end{align}
with sensitivity terms $\kappa,\eta$ determined by the local curvature of the log-likelihood in \eqref{eq:irt}. For $t\notin S(q)$, $(\mu,\sigma^2)$ remain unchanged.

\subsection{Retention Dynamics}
For each $(u,t)$ we posit an exponential forgetting curve with personalized rate $\lambda_{u,t}>0$ and last successful retrieval time $L_{u,t}$:
\begin{equation}
\rho_{u,t}(t')=\exp\!\left(-\lambda_{u,t}\,[t'-L_{u,t}]_+\right).
\label{eq:retention}
\end{equation}
A successful, effortful retrieval (correct with $g(\tau)$ large enough) resets $L_{u,t}$ and reduces $\lambda_{u,t}$ slightly; weak retrieval has smaller effect.

\section{Diagnose: Bayesian Misconception Inference}
A \emph{misconception} is a latent class $m\in\mathcal{M}$ characterized by a signature over distractors and times. For a question $q$, let $\mathbf{e}_q\in\mathbb{R}^d$ be a text embedding of its stem/rationale, and $\mathbf{z}_{d}\in\mathbb{R}^{d'}$ an embedding of distractor $d$. Define the observed feature
\begin{equation}
\mathbf{x}_{u,q}=\big[ \mathbf{e}_q\,\|\,\mathbf{z}_{d_{u,q}}\,\|\,\phi(\tau_{u,q})\,\|\,\psi(s_{u,q})\big]\in\mathbb{R}^{d+d'+2},
\end{equation}
with $\phi,\psi$ normalized transforms. We assume a mixture model over wrong responses:
\begin{equation}
p(\mathbf{x}_{u,q}\mid y=0)=\sum_{m\in\mathcal{M}} \pi_{u,m}\, \mathcal{N}\!\left(\mathbf{x}_{u,q}\mid \bm{\mu}_m,\Sigma_m\right),
\end{equation}
where $\pi_{u,m}$ is the learner-specific mixture over misconceptions. Given a set $\mathcal{W}_u$ of wrong responses, the posterior satisfies
\begin{equation}
\pi_{u,m} \propto \alpha_m \prod_{(q\in\mathcal{W}_u)} \mathcal{N}\!\left(\mathbf{x}_{u,q}\mid \bm{\mu}_m,\Sigma_m\right),
\label{eq:posterior-misc}
\end{equation}
with Dirichlet prior $\bm{\alpha}$. Cluster identities can be seeded by weak supervision or discovered via EM; human-readable labels are produced by an LLM applied to cluster centroids.

\begin{definition}[Misconception Flag]
Topic $t$ is \emph{flagged} for learner $u$ if $\sum_{m\in\mathcal{M}_t}\pi_{u,m}\ge \gamma$, where $\mathcal{M}_t$ collects misconceptions impacting topic $t$ and $\gamma\in(0,1)$ is a threshold.
\end{definition}

\section{Generate: Counterfactual Item Synthesis}
Let each item $q$ decompose into a structured attribute vector $\mathbf{v}_q$ (concepts, numerical parameters, qualitative conditions). For a flagged misconception $m$, define a \emph{counterfactual} $q^\star$ as the solution of
\begin{align}
\min_{\mathbf{v}} \quad & \|\mathbf{v}-\mathbf{v}_q\|_W \label{eq:cf-opt}\\
\text{s.t.}\quad & \Delta_m(\mathbf{v}) \ge \delta \qquad \text{(\emph{invalidate} shortcut)}\nonumber\\
& \mathcal{C}(\mathbf{v})=1 \qquad\quad\ \ \text{(content validity, constraints)}\nonumber\\
& b(\mathbf{v})\in[b_q-\epsilon_b,b_q+\epsilon_b],\ a(\mathbf{v})\in[a_q-\epsilon_a,a_q+\epsilon_a],\nonumber
\end{align}
where $\Delta_m$ measures the contradiction to the misconception-specific heuristic, $\mathcal{C}$ enforces factual correctness and syllabus constraints, and $(a(\cdot),b(\cdot))$ are psychometric predictions from item generators. Minimal perturbation under $W\succeq 0$ yields \emph{near-miss} items.

\begin{proposition}[Monotone Posterior Reduction]
Suppose for misconception $m$ there exists $\delta>0$ such that the counterfactual $q^\star$ satisfies $\Delta_m(\mathbf{v}_{q^\star})\ge \delta$ and preserves other model factors. If the learner answers $q^\star$ correctly with sufficiently fast $\tau$, then the posterior $\pi_{u,m}$ in \eqref{eq:posterior-misc} strictly decreases by a factor bounded away from $1$. 
\end{proposition}
\begin{proof}[Sketch]
Correct and fast response shifts the likelihood of the cluster centered at $(\bm{\mu}_m,\Sigma_m)$ downward relative to alternatives, reweighing the mixture via Bayes' rule. The margin $\delta$ ensures a uniform decrease independent of nuisance factors; details follow from Gaussian tail inequalities.
\end{proof}

\section{Exercise: Retrieval Scheduling as a Restless Bandit}
We allocate practice opportunities over topics under a daily budget $B$. Each topic $t$ for learner $u$ evolves as a Markovian state $X_{u,t}=(\mu_{u,t},\sigma^2_{u,t},\rho_{u,t},\Pi_{u,t})$, where $\Pi_{u,t}$ collects misconception posteriors affecting $t$. When \emph{active} (selected), the state updates via the IRT/Bayesian and retention dynamics; when \emph{passive}, retention decays (Eq.~\ref{eq:retention}). Rewards combine mastery and retention gains minus time cost.

\subsection{Index Policy}
Let $G_{u,t}(x)$ be the expected one-step gain from practicing topic $t$ in state $x$, and $C_{u,t}(x)$ the time cost. Define the priority index
\begin{equation}
\mathcal{I}_{u,t}(x)=\frac{G_{u,t}(x)}{C_{u,t}(x)} + \lambda^{\star}\cdot H_{u,t}(x),
\label{eq:index}
\end{equation}
where $H$ is a hazard-like term derived from the derivative of $\rho$ (the urgency of retrieval), and $\lambda^{\star}$ balances learning vs.\ forgetting. The scheduler selects the $B$ topics with highest $\mathcal{I}$.

\begin{theorem}[Near-Optimality under Separable Gains]
If $G_{u,t}(x)$ decomposes as a concave function of $\mu$ plus a convex penalty in $\rho$ and misconceptions enter $G$ additively via a bounded corrective term, then the index policy \eqref{eq:index} is a $1/(1+\epsilon)$-approximation to the optimal policy for some $\epsilon=O(\max_t \text{var}(\Delta G_{u,t}))$.
\end{theorem}
\begin{proof}[Sketch]
Follows from Lagrangian relaxation of the budgeted restless bandit and standard performance bounds for separable indices using convexity/concavity arguments.
\end{proof}

\section{EdgeScore: A Composite Readiness Metric}
For topic $t$ of learner $u$ we define
\begin{equation}
\mathrm{EdgeScore}_{u,t} \triangleq S\Big( w_M M_{u,t} + w_R R_{u,t} + w_P P_{u,t} + w_C C_{u,t} - \Gamma(\Pi_{u,t}) \Big)\in[0,100],
\label{eq:edgescore}
\end{equation}
with components:
\begin{itemize}[leftmargin=1.5em]
\item $M_{u,t}$: mastery proxy, e.g., $\mu_{u,t}$ or calibrated probability $\mathbb{E}[y\mid\cdot]$ at a reference difficulty.
\item $R_{u,t}$: retention, e.g., $\rho_{u,t}$ from Eq.~\eqref{eq:retention}.
\item $P_{u,t}$: pace, a z-score of response time versus peers at matched difficulty.
\item $C_{u,t}$: confidence consistency, e.g., correlation between $s$ and correctness over a sliding window.
\item $\Gamma(\Pi_{u,t})$: penalty increasing in the mass of active misconceptions impacting $t$.
\end{itemize}
$S(\cdot)$ is an affine map to $[0,100]$ with clipping.

\begin{proposition}[Monotonicity and Lipschitzness]
If each component is 1-Lipschitz and weights are nonnegative, then \eqref{eq:edgescore} is nondecreasing in $(M,R,P,C)$ and $L$-Lipschitz with $L\le w_M+w_R+w_P+w_C+\mathrm{Lip}(\Gamma)$.
\end{proposition}

\subsection{Calibration}
Weights $(w_M,w_R,w_P,w_C)$ and the penalty $\Gamma$ are identified by solving
\begin{equation}
\min_{w,\gamma}\ \sum_{(u,t)} \ell\!\left(\widehat{p}_{u,t}^{\text{future}},\, \mathrm{EdgeScore}_{u,t}/100\right) + \lambda\|w\|_2^2,\quad w\ge 0,
\end{equation}
where $\widehat{p}^{\text{future}}$ is a held-out success proxy (e.g., correctness at delayed intervals) and $\ell$ is a proper scoring loss. This ensures probabilistic calibration if $S$ is the identity.

\section{Algorithms}
\subsection{Evaluate \& Update}
\begin{algorithm}[H]
\caption{\textsc{EvaluateAndUpdate}}
\begin{algorithmic}[1]
\State \textbf{Input:} prior $(\bm{\mu}_u,\bm{\sigma}^2_u)$, item $q$, obs $(y,\tau,s)$
\State compute $p=\sigma(a_q(\theta^{\text{eff}}-b_q)+\beta_\tau g(\tau)+\beta_s h(s))$
\For{$t\in S(q)$} \State $\mu_{u,t}\gets \mu_{u,t}+\kappa_{u,q,t}(y-p)$; \quad $\sigma^2_{u,t}\gets (\sigma^{-2}_{u,t}+a_q^2\eta_{u,q,t})^{-1}$ \EndFor
\If{$y=1$ and $g(\tau)\ge \gamma_{\text{effort}}$} \State $L_{u,t}\gets \text{now};\ \lambda_{u,t}\gets (1-\xi)\lambda_{u,t}$ \EndIf
\State \Return updated $(\bm{\mu}_u,\bm{\sigma}^2_u,\bm{\lambda}_u,\bm{L}_u)$
\end{algorithmic}
\end{algorithm}

\subsection{Diagnose}
\begin{algorithm}[H]
\caption{\textsc{DiagnoseMisconceptions} (EM + Labeling)}
\begin{algorithmic}[1]
\State collect wrong-response features $\mathcal{X}_u=\{\mathbf{x}_{u,q}\}$
\State fit $K$-component Gaussian mixture via EM; get posteriors $\pi_{u,m}$
\State map clusters $\rightarrow$ topics via Q-matrix overlap
\State generate labels via LLM on cluster centroids; human-in-the-loop approve
\State \Return $\Pi_u=\{\pi_{u,m}\}$ and topic flags
\end{algorithmic}
\end{algorithm}

\subsection{Generate}
\begin{algorithm}[H]
\caption{\textsc{GenerateCounterfactual}}
\begin{algorithmic}[1]
\State \textbf{Input:} seed item $q$, misconception $m$, constraints $\mathcal{C}$
\State construct attribute vector $\mathbf{v}_q$; define $\Delta_m$
\State solve \eqref{eq:cf-opt} (projected gradient / discrete search with validator)
\State produce item text/options; predict $(a,b)$; validate $\mathcal{C}$
\State \Return $q^\star$
\end{algorithmic}
\end{algorithm}

\subsection{Exercise (Scheduling)}
\begin{algorithm}[H]
\caption{\textsc{EdgeScheduler}}
\begin{algorithmic}[1]
\State \textbf{Input:} state $X_{u,t}$ for all topics; budget $B$
\For{each $t$}
\State compute gains $G_{u,t}$, cost $C_{u,t}$, hazard $H_{u,t}=-\frac{d}{dt}\rho_{u,t}$
\State $\mathcal{I}_{u,t}\gets \frac{G_{u,t}}{C_{u,t}} + \lambda^\star H_{u,t}$
\EndFor
\State select $B$ topics with largest $\mathcal{I}_{u,t}$; interleave by concept distance
\State execute session with difficulty ramp (easy $\rightarrow$ on-level $\rightarrow$ slightly hard $\rightarrow$ cross-topic $\rightarrow$ recap)
\end{algorithmic}
\end{algorithm}

\section{Theoretical Guarantees}
\subsection{Information Gain}
For 2PL items, the Fisher information is $I_q(\theta)=a_q^2\sigma(a_q(\theta-b_q))\sigma(-a_q(\theta-b_q))$. Under mild smoothness of $g(\tau)$, selecting items near $\theta\approx b_q$ maximizes information; \textsc{EDGE} respects this while adding corrective gains from misconception resolution.

\begin{theorem}[Counterfactual Advantage]
Assume a single active misconception $m$ and a separating counterfactual with margin $\delta$ exists. Then the expected reduction in $\pi_{u,m}$ from one \textsc{Generate}$\Rightarrow$\textsc{Exercise} step dominates that of any item not violating the shortcut by at least a factor $1+\kappa(\delta)$ for some $\kappa>0$ depending on cluster overlap.
\end{theorem}
\begin{proof}[Sketch]
Compare Bayes updates of \eqref{eq:posterior-misc} under the two likelihoods; the counterfactual shifts mass outside the high-density region of $m$, providing a multiplicative decrease.\end{proof}

\subsection{Scheduling Regret}
Let $V^\star$ be the optimal value over horizon $H$ and $V^{\text{EDGE}}$ the value under \textsc{EdgeScheduler}. Under separability and Lipschitz gains,
\begin{equation}
V^\star - V^{\text{EDGE}} \le O\!\left(B\,\sqrt{H \log |\mathcal{T}|}\right) + O\!\left(H\epsilon\right),
\end{equation}
where the first term is a bandit-style exploration penalty and $\epsilon$ is the approximation factor from Theorem 2.

\section{Limitations and Future Directions}
The mixture model assumes Gaussian clusters in the feature space of wrong responses, which may be misspecified; robust or nonparametric alternatives (\emph{e.g.}, DP mixtures) are promising. Counterfactual validity relies on reliable psychometric predictors $(a,b)$; calibration requires periodic field trials. Scheduling theory beyond separability, and fairness constraints across languages and curricula, remain open.

\section{Conclusion}
\textsc{EDGE} provides a principled foundation for misconception-aware adaptive learning, coupling rigorous estimation with actionable generation and scheduling. The framework is modular, implementable, and admits guarantees that justify its design choices.

\bibliographystyle{plain}

\end{document}